\documentclass{article}

\usepackage[final]{exait_2025}

\usepackage[utf8]{inputenc} % allow utf-8 input
\usepackage[T1]{fontenc}    % use 8-bit T1 fonts
\usepackage{hyperref}       % hyperlinks
\usepackage{url}            % simple URL typesetting
\usepackage{booktabs}       % professional-quality tables
\usepackage{amsfonts}       % blackboard math symbols
\usepackage{nicefrac}       % compact symbols for 1/2, etc.
\usepackage{microtype}      % microtypography
\usepackage{xcolor}         % colors

\usepackage{natbib}
\usepackage{amsmath}
\usepackage{amssymb}
\usepackage{mathtools}
\usepackage{amsthm}
\usepackage{multirow}

\usepackage{algorithm}
\usepackage{algpseudocode}

\usepackage[capitalize,noabbrev]{cleveref}

\theoremstyle{plain}
\newtheorem{theorem}{Theorem}[section]

\theoremstyle{definition}
\newtheorem{definition}[theorem]{Definition}

\newtheorem{assumption}[theorem]{Assumption}
\theoremstyle{remark}

\usepackage{thm-restate}

\makeatletter
\def\thmt@innercounters{equation,theorem}
\makeatother

\usepackage[colorinlistoftodos]{todonotes}
\definecolor{Nikola}{rgb}{0, 0.0, 1.0}
\definecolor{Johannes}{rgb}{0, 0.0, 1.0}

\usepackage{mathrsfs}
\usepackage[mathscr]{euscript}
\DeclareSymbolFont{rsfs}{U}{rsfs}{m}{n}
\DeclareSymbolFontAlphabet{\mathscrsfs}{rsfs}

\newcommand{\C}{{\operatorname{C}}}

\newcommand{\bS}{\mathcal{S}}
\newcommand{\A}{\mathcal{A}}

\usepackage{amsmath,amsfonts,bm}

\def\eqref#1{equation~\ref{#1}}
\def\Eqref#1{Equation~\ref{#1}}
\def\1{\bm{1}}

\DeclareMathAlphabet{\mathsfit}{\encodingdefault}{\sfdefault}{m}{sl}
\SetMathAlphabet{\mathsfit}{bold}{\encodingdefault}{\sfdefault}{bx}{n}

\DeclareMathOperator*{\argmax}{arg\,max}

\usepackage{comment}
\usepackage{multicol}

\usepackage{hyperref}
\usepackage{cleveref}
\usepackage{url}

\usepackage{subcaption}
\usepackage{booktabs}

\usepackage{wrapfig}

\usepackage{adjustbox}

\definecolor{cbs1}{HTML}{00A89D}

\title{Central Path Proximal Policy Optimization} %for Constrained Markov Decision Processes}

\author{Nikola Milosevic \\
Max Planck Institute for Human Cognitive and Brain Sciences, Leipzig\\
Center for Scalable Data Analytics and Artificial Intelligence (ScaDS.AI), Dresden/Leipzig\\
\texttt{nmilosevic@cbs.mpg.de} \\
\And
Johannes M\"uller \\
Institut f\"{u}r Mathematik, Technische Universit\"{a}t Berlin, 10623 Berlin, Germany \\
\AND
Nico Scherf \\
Max Planck Institute for Human Cognitive and Brain Sciences, Leipzig\\
Center for Scalable Data Analytics and Artificial Intelligence (ScaDS.AI), Dresden/Leipzig\\
}

\begin{document}

\maketitle

\begin{abstract}
In constrained Markov decision processes, enforcing constraints during training is often thought of as decreasing the final return. 
Recently, it was shown that constraints can be incorporated directly into the policy geometry, yielding an optimization trajectory close to the central path of a barrier method, which does not compromise final return.
Building on this idea, we introduce Central Path Proximal Policy Optimization (C3PO), a simple modification of the PPO loss that produces policy iterates, that stay close to the central path of the constrained optimization problem.
Compared to existing on-policy
methods, C3PO delivers improved performance with tighter constraint enforcement, suggesting that central path-guided updates offer a promising direction for constrained policy optimization.
\end{abstract}

\section{Introduction}

Reinforcement learning (RL) has demonstrated impressive capabilities across a wide range of domains, yet real-world applications increasingly demand more than just reward maximization. 
In many real-world high-stakes environments
%—such as autonomous vehicles, healthcare, or robotic manipulation—
agents must also avoid violating domain-specific safety or resource constraints. 
This motivates the study of constrained Markov decision processes (CMDPs), an extension of the standard RL framework that imposes expected cost constraints alongside the goal of reward maximization~\citep{Altman1999ConstrainedMD}. 
By treating feasibility and reward objectives separately, CMDPs provide a principled framework for specifying agent behavior in complex environments.

CMDPs are especially relevant in deep reinforcement learning settings, where the design of reward functions 
%that reliably encode which states to avoid 
that lead to safe behavior is difficult. 
Prior work has emphasized the importance of explicit constraint modeling in reinforcement learning for safe exploration~\citep{ray2019benchmarking} and
%argue that CMDPs offer a natural formalism for benchmarking safe exploration in deep RL and introduce the Safety Gym suite to evaluate algorithms based on both task performance and cumulative safety cost. 
%Similarly, 
complex task specification~\citep{roy2022directbehaviorspecificationconstrained}, where
%demonstrate that task specifications expressed via 
constraints can be more natural and easier to design, including finetuning LLMs for harmlessness~\citep{dai2023saferlhfsafereinforcement}.

%While CMDPs provide a unified formalism, it is crucial to distinguish between two fundamentally different settings: (i) \emph{safe exploration}, where constraints must be satisfied throughout training, and (ii) \emph{safe convergence}, where only the final policy is required to meet the constraints. In the first case, exploration must be conservative, since unsafe actions during training may have irreversible consequences—particularly when learning is performed on real-world systems. In this setting, model-free on-policy methods are appealing due to their predictable updates and tight control over policy changes, which facilitate the enforcement of safety guarantees during learning. 
%However, they tend to be too sample inefficient to consider for real-world training from scratch, which has prompted a growing body of research in off-policy, offline, and model-based safe RL. model mismatch... complexity... Questionable guarantees and scalability...

Despite their relatively low sample efficiency, model-free on-policy algorithms continue to play a foundational role in constrained RL. 
They offer conceptual clarity, support rigorous theoretical analysis, and provide strong baselines for studying the balance between performance and constraint satisfaction. 
As the field moves toward more scalable and sample-efficient approaches, insights developed in the on-policy setting remain central to both algorithm design and our broader understanding of safe learning, such as the policy improvement guarantees and constraint violation bounds introduced by \cite{achiam2017constrained}.

In this context, there is a growing need for simple, scalable, and effective algorithms for solving CMDPs, ideally with properties similar to widely used algorithms such as proximal policy optimization (PPO;~\cite{schulman2017proximalpolicyoptimizationalgorithms}). PPO's robustness, ease of implementation, and scalability have made it the method of choice in many deep RL and RLHF pipelines~\citep{ouyang2022training}. We aim to extend these strengths to the constrained setting by developing an algorithm that shares PPO's practical benefits while enforcing constraints in a principled CMDP framework. Specifically, we seek to achieve high final reward while approximately satisfying constraints, at least at convergence.

To frame this problem, we distinguish between two commonly conflated settings in constrained RL: (i) \emph{safe exploration}, where constraints must be satisfied throughout training, and (ii) \emph{safe convergence}, where only the final policy is required to satisfy the constraints. Much of the literature has focused on the former, motivated by safety-critical applications in the real world. 
The dominant approach in this setting is model-based safe RL, which can provide strong safety and stability guarantees~\citep{berkenkamp2017safeModelBased, as2025actsafe}.
%where exploration itself must be safe. 
In contrast, safe convergence reflects settings like simulation-based training or alignment finetuning~\citep{dai2023saferlhfsafereinforcement}.
%, where the cost of unsafe behavior during training is tolerable. 
Typically, ensuring safety \emph{during} training is considered to decrease the final performance achieved by an algorithm. 
%We hypothesize that even in this regime, enforcing feasibility throughout training %(as is the approach of on-policy (barrier-)panalty methods) 
We show the contrary and present an algorithm that exhibits strict feasibility during training as well as reliable feasibility and high return at convergence. 
%We will refer to this as the \emph{central path hypothesis (CPH)}. 
\begin{wrapfigure}{rt}{0.5\textwidth}
  \begin{center}
    \includegraphics[width=0.48\textwidth]{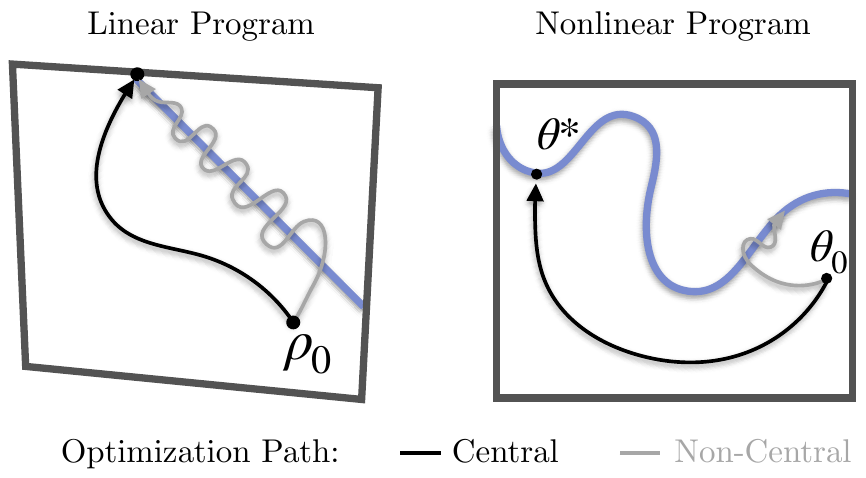}
  \end{center}
  \caption{Pictorial visualization of the central path argument from the main text. While a wide range of methods technically converge to an optimal feasible solution in the linear programming formulation of finite CMDPs (left), in the function approximation setting (right), approaching the constraint surface too early may result in higher sensitivity to local optima.}
  \label{fig:c3po-idea}
  \vspace{-13pt}
\end{wrapfigure}

%The CPH rests on the following geometric insight. 

In nonlinear CMDPs, the constraint surface is typically curved and nonconvex in policy space. 
Converging prematurely or oscillating near the constraint boundary during training
%—as is often the case with Lagrangian or most trust region methods—
can lead to unreliable constraint satisfaction at convergence. 
Furthermore, it can lead the iterates to local optima that satisfy the constraints but fail to achieve high reward, see Figure~\ref{fig:c3po-idea}. 
Penalty and barrier methods address this problem by maintaining a feasible trajectory toward the constraint surface, yielding feasible solutions more reliably. 
However, barrier methods introduce bias~\citep{muller2024essentially}, meaning the optimization problem obtained by adding a barrier penalty does not have the same solution set as the original problem, which can lead to degraded reward in policy optimization~\citep{anonymous2025embedding}.
Barrier methods either require careful tuning or an interior point approach~\citep{liu2020ipo} to avoid harming reward performance. 

The recently proposed C-TRPO~\citep{anonymous2025embedding} addresses these challenges by combining the strengths of trust-region and barrier methods by deriving a barrier-inspired trust-region formulation using strictly feasible trust regions. 
This results in an algorithm that acts like a barrier method with an adaptively receding barrier, introducing no regularization bias as a result. 
This is achieved by producing policies, which are close to the regularization path obtained by altering the regularization strength, which is commonly known as the \emph{central path}~\citep{boyd2004convex}.
%This introduces no bias in the optimal feasible solution but still produces strictly feasible policies just like a barrier method. 
%Further, C-TRPO produces policies, which are close to the regularization path obtained by altering the regularization strength, which is commonly known as the \emph{central path}~\cite{boyd2004convex, muller2024essentially}. 
%This nicely illustrates how the constraints are incorporated in the algorithm's geometry and ensures that C-TRPO and variants thereof produce policies which don't prematurely approach the constraint surface. 

C-TRPO's scalability is limited due to computational overhead 
%and sensitivity to large network architectures or batch sizes 
introduced by the TRPO-inspired update, and the update is defined only in the feasible set.
%This motivates us to design a first-order variant of C-TRPO closer to a proximal algorithm
%An advantage of C-TRPO seems to be its parallels to the original TRPO algorithm: For small step sizes, it approaches a natural policy gradient update called C-NPG, which shares some key properties with the original NPG~\cite{kakade2001natural}. 
%One such property is the \emph{central path property} stating that natural policy gradients produce policies that stay close to the central path of entropy regulatization, meaning that they are close to entropy regularization optimal policies with decreasing regularization parameter~\cite{muller2024essentially}. 
%In the context of constrained MDPs this ensures that C-TRPO and variants thereof produce policies that are . 
%We believe that this is implicitly exploited by C-TRPO, and discuss the property more precisely in Appendix \ref{app:background}. 
%
%To test the CPH empirically, and 
To address the need for a simple and scalable CMDP solver, we propose a proximal version of C-TRPO. 
It also %approximately 
follows the central path, and consequently we call it \emph{Central Path Proximal Policy Optimization (C3PO)}. 
C3PO is a minibatch-based method that approximates the C-TRPO update using an exact penalty formulation, combining the simplicity and efficiency of PPO-style updates with the feasible geometry of central path methods. At its core, C3PO leverages the central path property of natural policy gradients to gradually guide the policy toward the constraint surface without inducing oscillations or premature convergence. 
%The result is a practical algorithm that retains high reward performance while satisfying constraints at convergence, and which has the potential to scale well to large neural networks and modern deep RL settings.

\section{Background}\label{sec:background}
We consider the infinite-horizon discounted constrained Markov decision process (CMDP) and refer the reader to \cite{Altman1999ConstrainedMD} for a general treatment. 
The CMDP is given by the tuple $\mathcal{M}\cup\mathcal{C}$, consisting of a finite MDP $\mathcal{M}$ and a set of constraints $\mathcal{C}$. 
The finite MDP $\mathcal{M}=\{\bS, \A, P, r, \mu, \gamma\}$ is defined by a finite state-space $\bS$, a finite action-space $\A$, a transition kernel $P\colon \bS \times \A \rightarrow \Delta_\bS$, an extrinsic reward function $r\colon \bS \times \A\rightarrow \mathbb{R}$, an initial state distribution $\mu \in \Delta_\bS$, and a discount factor $\gamma\in [0, 1)$.
The space $\Delta_\bS$ is the set of categorical distributions over $\bS$. Further, $\mathcal{C}=\{(c_i,b_i)\}_{i=1}^m$ is a set of $m$ constraints, where $c_i\colon\bS\times\A\to\mathbb{R}$ are the cost functions and $d_i \in \mathbb{R}$ are the cost thresholds.
An agent interacts with the CMDP by selecting a policy $\pi\in\Pi$ and collecting trajectories $\tau=(s_0,a_0,...s_T)$. Let $R(\tau) = \sum_{t=0}^{\infty} \gamma^t r(s_t, a_t)$. Given $\pi$, the value function, 
%V_r^\pi\colon \bS \to \mathbb{R}$
action-value function,
%$Q_r^\pi\colon \bS\times\A \to \mathbb{R}$
and advantage function
%$A_r^\pi\colon \bS \times \A \to \mathbb{R}$ 
associated with the reward function $r$ are defined respectively as
\begin{equation*}
    V_r^\pi(s) \coloneqq (1-\gamma)\,\mathbb{E}_{\tau\sim\pi} \left[ R(\tau) \Big|  s_0 = s \right],    \quad
%\end{equation*}
%\begin{equation*}
    Q_r^\pi(s, a) \coloneqq (1-\gamma)\,\mathbb{E}_{\tau\sim\pi} \left[ R(\tau) \Big| s_0 = s, a_0 = a \right],
\end{equation*}
and
\begin{equation*}
A_r^\pi(s,a) \coloneqq Q_r^\pi(s, a) - V_r^\pi(s).
\end{equation*}
The expectations are taken over trajectories of the Markov process, meaning with respect to the initial distribution $s_0\sim\mu$, the policy $a_{t}\sim\pi(\cdot|s_t)$ and the state transition $s_{t+1} \sim P(\cdot|s_t, a_t)$. $V_{c_i}^\pi(s)$, $Q_{c_i}^\pi(s,a)$ and $A_{c_i}^\pi(s,a)$ are defined analogously for the $i$-th cost $c_i$ instead of $r$.

Constrained reinforcement learning addresses the optimization problem
\begin{equation}\label{eq:CMDP}
    \text{maximize}_{\pi \in \Pi} \; R(\pi) \quad \text{subject to} \quad C_i(\pi) \leq d_i %\quad\text{for all } 
\end{equation}
for all $i=1, \dots, m$,
where $R(\pi)$ is the expected value under the initial state distribution
$R(\pi) \coloneqq  \mathbb{E}_{s \sim \mu}[V_r^\pi(s)]$ and $C_i(\pi) \coloneqq  \mathbb{E}_{s \sim \mu}[V_{c_i}^\pi(s)]$.

Every stationary policy $\pi$ induces the occupancy measures $
\rho_\pi(s)\coloneqq (1-\gamma)\sum_{t=0}^\infty \gamma^t \mathbb{P}_\pi (s_t = s)$, and $\rho_\pi(s,a)\coloneqq\rho_\pi(s)\pi(a|s)$
which indicate the relative frequencies of visiting a state(-action) pair, discounted by how far the event lies in the future. 
The classical, linear programming (LP) approach to solving finite CMDPs \citep{Altman1999ConstrainedMD}, reformulates problem \ref{eq:CMDP} as
\begin{equation}\label{eq:CMDP-lin}
    \text{maximize}_{\rho \in \mathcal{K}} \; \sum_{s,a} \rho(s,a) r(s,a) \quad \text{subject to} \quad \sum_{s,a} \rho(s,a) c_i(s,a) \leq d_i %\quad\text{for all } 
\end{equation}
which can be solved using LP solution methods to obtain an optimal occupancy measure $\rho^*$. Here, $\mathcal{K}$ is a set of linear constraints that $\rho_\pi$ must satisfy~\citep{kallenberg1994survey, mei2020escaping}, sometimes referred to as the Bellman flow equations. Finally, an optimal policy can be extracted by conditioning $\pi^*(a|s)=\rho^*(s,a)/\sum_{a'}\rho^*(s,a')$.

In the function approximation setting, approach \ref{eq:CMDP-lin} is not applicable, which has prompted a large body of research in on-policy policy optimization methods. However, it can be leveraged to derive general constrained RL algorithms~\citep{anonymous2025embedding}. In the analysis of on-policy methods (including for standard MDPs) the \emph{policy advantage} plays an important role. In CMDPs, the policy advantages are defined as
\begin{equation}
    \mathbb{A}_{r}^{\pi_k}(\pi)=\sum_{s,a}\rho_{\pi_k}(s)\pi(a|s)A_{r}^{\pi_k}(s,a) \quad \text{and} \quad \mathbb{A}_{c}^{\pi_k}(\pi)=\sum_{s,a}\rho_{\pi_k}(s)\pi(a|s)A_{c}^{\pi_k}(s,a).
\end{equation}
They play an important role in policy optimization, as they approximate the performance difference between two nearby policies with respect to the reward $\mathbb{A}_{r}^{\pi_k}(\pi) \approx R(\pi) - R(\pi_k)$ if $\pi_k \approx \pi$ and analogously for the cost function.

\subsection{Policy Optimization Methods for Constrained Reinforcement Learning}

In the following, we review relevant prior constrained policy optimization methods, thereby focusing on a single constraint to reduce notational clutter. However, all mentioned methods are trivial to extend to multiple constraints.

\paragraph{Constrained Policy Optimization (CPO)}
Constrained policy optimization (CPO) is a modification of trust region policy optimization (TRPO; \cite{schulman2017trust}), where the classic trust region is intersected with the set of safe policies~\citep{achiam2017constrained}. 
At each iteration $k$, the policy of the next iteration $\pi_{k+1}$ is obtained through the solution of 
\begin{equation}\label{eq:cpo}
\max_{\pi \in \Pi}\ \mathbb{A}_{r}^{\pi_k}(\pi)
            \ \textrm{ s.t. }\ \bar D_\textup{KL}(\pi,\pi_k) \leq \delta \ \textrm{ and }\ C(\pi_k) + \mathbb A^{\pi_k}_c(\pi) \le d.
\end{equation}
where $\bar D_\textup{KL}(\pi,\pi_k) = \sum_{s,a}\rho_{\pi_k}(s)D_{\textup{KL}}[\pi(\cdot|s)|\pi_k(\cdot|s)]$ and $C(\pi_k) + \mathbb A_c^{\pi_k}(\pi)$ is an estimate for $C(\pi)$, see~\cite{kakade2002Approximately, schulman2017trust, achiam2017constrained}. 

\paragraph{Penalized Proximal Policy Optimization (P3O/P2BPO)}
Solving the constrained optimization problem \eqref{eq:cpo} is difficult to scale up to more challenging tasks and larger model sizes, as it relies on the arguably sample inefficient TRPO update.
To circumvent this~\cite{zhang2022penalizedproximalpolicyoptimization} proposed a Constrained RL algorithm derived from the relaxed penalized problem 
\begin{equation}\label{eq:p3o-penalty}
\max_{\pi \in \Pi}\ \mathbb{A}_{r}^{\pi_k}(\pi) - \lambda 
            \max\{0,  C(\pi_k) + \mathbb A_c^{\pi_k}(\pi) - d \},
            \ \textrm{ s.t. }\ \bar D_\textup{KL}(\pi,\pi_k) \leq \delta.
\end{equation}
The appeal of this reformulation is that one can obtain an unconstrained problem that gives the same solution set for $\lambda$ chosen large enough~\citep{zhang2022penalizedproximalpolicyoptimization} and by employing a PPO-like loss. A similar approach was taken by \cite{dey2024p2bpo}, where $\max\{0, \cdot\}$ was replaced with a softplus.

\paragraph{Constrained Trust Region Policy Optimization (C-TRPO)}
Where~\eqref{eq:cpo} incorporates constraints by intersecting the trust region with the set of safe policies, an alternative approach was taken by~\cite{anonymous2025embedding} where the geometry was modified such that the resulting trust region automatically consists of safe policies.
To this end, C-TRPO proceeds as TRPO but with the usual divergence augmented by a barrier term, meaning
\begin{equation}\label{eq:ctrpo}
\max_{\pi \in \Pi}\ \mathbb{A}_{r}^{\pi_k}(\pi)
            \ \textrm{ s.t. }\ \bar D_\textup{KL}(\pi,\pi_k) + \beta D_\textup{B}(\pi,\pi_k) \leq \delta.
\end{equation}
where we'll refer to
\begin{equation}
    D_\textup{B}(\pi,\pi_k) = \frac{b-\mathbb{A}_{c}^{\pi_k}(\pi)}{b} - \log\left(\frac{b-\mathbb{A}_{c}^{\pi_k}(\pi)}{b}\right) - 1, \textrm{ for } b>0 \textrm{, else } \infty
\end{equation}
as the \emph{barrier divergence}, %\footnote{Note the similarity to the unbiased KL-Divergence estimator $\hat D_{KL}=\frac{\pi(a|s)}{\pi_k(a|s)}-\log\frac{\pi(a|s)}{\pi_k(a|s)}-1$}, 
$\beta$ is a positive safety parameter, and $b=d-C(\pi_k)$ is the \emph{cost budget}.

This update is justified by the general theory of Bregman divergences and the theory of convex programs. It has desirable theoretical properties and results in state-of-the-art performance compared to other on-policy CMDP algorithms. 
We refer the reader to \cite{anonymous2025embedding} and Appendix \ref{app:background} for detailed discussions.

\paragraph{Other methods} 
So far we have focused on model-free, direct policy optimization methods, specifically trust-region and PPO-penalty based ones. 
However, it is important to note that model-based approaches, e.g. \cite{berkenkamp2017safeModelBased, as2025actsafe}, 
are also attractive due to their stability and safety guarantees,
but require learning a model, which is not always feasible.
%\paragraph{Lagrangian methods} 
Lagrangian methods are a widely adopted approach, where the CMDP is formulated as a primal-dual optimization problem. Optimizing the dual variable with stochastic gradient descent is a popular baseline~\citep{achiam2017constrained, ray2019benchmarking, chow2019lyapunovbasedsafepolicyoptimization, stooke2020responsivesafetyreinforcementlearning}. However, a naively optimized dual variable may cause oscillations and overshoot, and analyzing more nuanced update strategies is a subject of current research interest~\citep{sohrabi2024picontrollersupdatinglagrange}.
%\paragraph{Log-barrier methods}
More recently, log-barrier approaches have attracted considerable research interest \citep{usmanova2024log, zhang2024constrained, dey2024p2bpo} due to their algorithmic simplicity through the fixed penalty, but also due to recent rigorous treatments, see e.g. \cite{ni2024safe}.
However, working with an explicit penalty
produces suboptimal policies w.r.t the original constrained MDP. Fixed penalties introduce an additional error, which has to be controlled, see for example~\cite{geist2019theory, muller2024essentially} for theory, and \cite{liu2020ipo} for a practical example of regularization bias.
%In contrast, combining trust region-based updates as in TRPO~\cite{schulman2017trust} with constrained optimization techniques does not change the objective function and the set of optimizers, and therefore does not introduce an additional error. 

\clearpage\section{Central Path Proximal Policy Optimization}
C-TRPO has desirable properties but the practical algorithm 1) scales poorly and is sample-inefficient due to its reliance on the TRPO algorithm and 2) relies on a recovery mechanism (reward-free cost minimization + hysteresis), since the update is not defined if $\pi_k$ is outside the feasible set.
\begin{figure}
    \centering
    \includegraphics[width=1\linewidth]{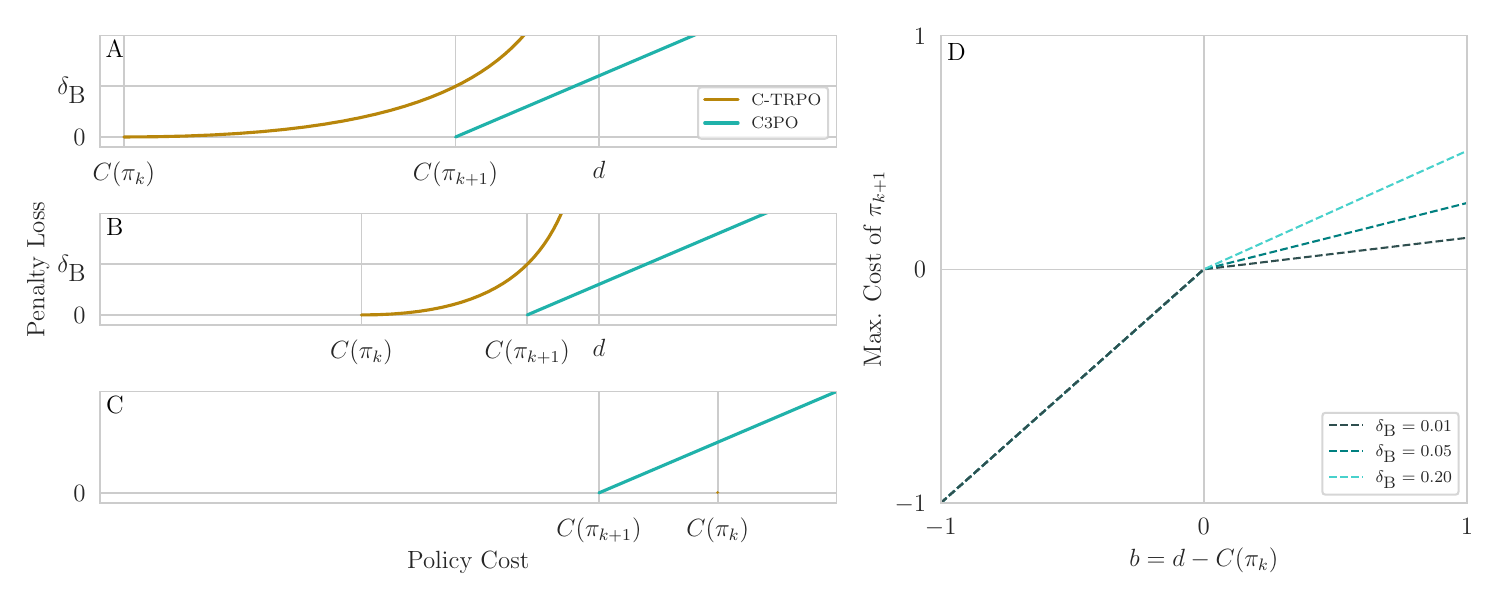}
    \caption{The working principle behind C3PO's exact penalty approach: As the iterate moves closer towards the constraint (A-C), C3PO's ReLU-penalty pulls away at a slower rate $0<w<1$, mimicking C-TRPO's barrier divergence. This rate is defined as a function of $\delta_\textup{B}$ (D), see main text. 
    While C-TRPO's barrier penalty is undefined if either $d \leq C(\pi_k)$ or $d \leq C(\pi_{k+1})$, C3PO's ReLU-penalty is defined everywhere.}
    \label{fig:moving-barrier}
\end{figure}

We propose a first-order approximation of C-TRPO that approximates its central path by solving surrogate optimization problems with the same solution set as C-TRPO's update on every iteration. %in the case where the CMDP is solvable from the interior. 
In addition, by employing an exact penalty approach, it allows unsafe policies during training, also enabling less strict exploration strategies within the safe convergence regime.

\paragraph{C3PO Update}
Let us consider a slight modification of C-TRPO's update, which is constrained with the KL and Barrier constraints separately, since they can be approximated using different methods which result in different precisions, i.e. we consider
\begin{equation}\label{eq:max-div}
\max_{\pi \in \Pi}\ \mathbb{A}_{r}^{\pi_k}(\pi)
            \ \textrm{ s.t. }\ D_\textup{B}(\pi,\pi_k) \leq \delta_\textup{B} \ \textrm{ and } \ \bar D_\textup{KL}(\pi,\pi_k) \leq \delta_\textup{KL}.
\end{equation}
Note that this is a subtly different problem than that posed by C-TRPO, but $\delta_{\textup{KL}}$ and $\delta_\textup{B}$ can always be chosen to include the feasible set entirely in C-TRPO's feasible set for a given $\delta$. 
Instead of solving this constrained problem directly, we consider the penalized problem given by 
\begin{equation}\label{eq:exact-penalty-formulation}
\max_{\pi \in \Pi}\ \mathbb{A}_{r}^{\pi_k}(\pi) - \kappa 
            \max\{0, D_\textup{B}(\pi,\pi_k) - \delta_{\textup{B}}\}
            \ \textrm{ s.t. }\ \bar D_\textup{KL}(\pi,\pi_k) \leq \delta_\textup{KL}.
\end{equation}

\begin{theorem}[Exactness]
Let $\lambda$ be the Lagrange multiplier vector for the optimizer of \Eqref{eq:max-div}. 
Then for $\kappa\ge \lvert \lambda \rvert$ the solution sets of problem \Eqref{eq:max-div} and problem \Eqref{eq:exact-penalty-formulation} agree. 
\end{theorem}
\begin{proof}
Note that the problem \Eqref{eq:max-div} is concave-convex in $\pi$. 
Hence, this is a special case of the general exactness result \Cref{app:thm:exactness}. 
\end{proof}

\paragraph{C3PO Algorithm}
The update \Eqref{eq:exact-penalty-formulation} is still undefined outside the feasible set of the barrier divergence constraint.
%\begin{equation}
%\begin{aligned}
%    L^{\text{PPO}} + \kappa \operatorname{ReLU}\left(\Psi\left(\mathbb{E}_{s,a\sim \rho_k} \left[\max\left(r_t(\theta) \hat{A}^{c_i}_t, \text{clip}(r_t(\theta), 1 - \epsilon, 1 + \epsilon) \hat{A}^{c_i}_t\right)\right]\right) - \delta\right).
%\end{aligned}
%\end{equation}
%
Since we use the barrier divergence only to define the feasible solution set of the update, we can replace it with another function, as long as it defines the same feasible set. 
More precisely, this can be achieved with an equivalent linear constraint that is zero where $D_\textup{B}(\pi,\pi_k) = \delta_\textup{B}$ for positive cost advantages.
The C3PO algorithm approximates update \ref{eq:exact-penalty-formulation} as
\begin{equation}\label{eq:exact-penalty-reformulation}
\max_{\pi \in \Pi}\ \mathbb{A}_{r}^{\pi_k}(\pi) - \kappa 
            \max\{0, \mathbb{A}_{c}^{\pi_k}(\pi) - \min\{b, w\cdot b\}\}
            \ \textrm{ s.t. }\ \bar D_\textup{KL}(\pi,\pi_k) \le \delta_{\textup{KL}}.
\end{equation}
where $0<w<1$.

\begin{restatable}[Positive Exactness]{proposition}{barrierVSlin}
\label{prop:barrier-vs-lin}
For $0\leq\mathbb{A}_c^{\pi_k}(\pi)<d-C(\pi_k)$, there exist $w$ and $\delta_\textup{B}$ for which the solution sets of problems \ref{eq:max-div}, \ref{eq:exact-penalty-formulation} and \ref{eq:exact-penalty-reformulation} agree.
\end{restatable}

The new update expresses the same constraint using a linear ReLU-penalty. The rate $w$ is a new hyper-parameter and we refer to Appendix \ref{app:derivation} for a proof of \Cref{prop:barrier-vs-lin}. 
Since the original problem's penalty function is not defined outside the interior of the feasible set, we must handle the case $C(\pi_k)\geq d$ differently, which is taken care of by the $\min(b, \cdot)$ term: For $b<0$, problem \ref{eq:exact-penalty-reformulation} reduces to
the P3O~\citep{zhang2022penalizedproximalpolicyoptimization} objective Eq. \ref{eq:p3o-penalty}.
Finally, the additional KL-constraint is approximated as in PPO~\citep{schulman2017proximalpolicyoptimizationalgorithms}. 
The resulting loss only consists of the PPO loss and an additional loss term which is a function of the policy cost advantage estimate. 
Let $r(\theta)=\frac{\pi_\theta(a|s)}{\pi_k(a|s)}$ denote the likelihood ratio of the optimized and last behavior policies and let
\begin{equation}\alpha_\text{clipped}(\theta) = \mathbb{E}_{s,a\sim \rho_k} \left[ \max\left(r(\theta) \hat{A}_{c}(s,a), \text{clip}(r(\theta), 1 - \epsilon, 1 + \epsilon) \hat{A}_{c}(s,a) \right)\right].\end{equation}
The C3PO loss is
\begin{equation}
\begin{aligned}\label{eq:c3po-loss}
    L^{\text{C3PO}}(\theta) = \operatorname{ReLU}\left(\alpha_\text{clipped}(\theta)-\min\{b, w\cdot b\}\right).
\end{aligned}
\end{equation}

The penalty coefficient remains a hyperparameter, which can be flexibly scheduled to solve CMDPs in the safe convergence regime, as shown in Section \ref{sec:emperiments}, where we use a linear schedule to achieve high final performance across multiple tasks. 
The final method is summarized in Algorithm \ref{alg:c3po}.

\begin{algorithm}[h!]
\caption{C3PO (deviation from PPO {\color{cbs1} in green})}\label{alg:c3po}
\begin{algorithmic}[1]
\Require Initial policy $\pi_0$ and value functions $\hat V_{r},{\color{cbs1}\hat V_{c_i}}${\color{cbs1}, thresholds $d_i$, scheduled penalty $\kappa_k$, rate $w$}
\For{$k = 0, 1, 2, \ldots$}
    \State Collect trajectory data $\mathcal{D} = \{s_0, a_0, r_0, c_0, \ldots\}$ by running $\pi_k$
    \State Estimate reward advantage $\hat{A}^r_t$ and {\color{cbs1}cost advantages $\hat{A}^{c_i}_t$} using GAE-$\lambda$ \citep{schulman2018highdimensionalcontinuouscontrolusing}
    %\State Compute importance ratio $r(\theta) = \frac{\pi_\theta}{\pi_k}(a|s)$ and budget $b_{\pi_k}=(1-\gamma)(d - C(\pi_k))$
    \State Update policy $\pi_{k+1}$ by minimizing $L^{\text{PPO}} {\color{cbs1} + \kappa_kL^{\text{C3PO}}}$ (\Eqref{eq:c3po-loss})
    \State Update value function estimates $\hat V_r^{\pi_{k+1}}$ and $\hat V_{c_i}^{\pi_{k+1}}$ by regression
\EndFor
\end{algorithmic}
\end{algorithm}

\paragraph{Relation to other PPO-Penalty methods}
C3PO is a superset of P3O~\citep{zhang2020first}. 
More precisely, if we set $w=1$ in C3PO, we obtain the P3O loss exactly. 
%This is analogous to how C-TRPO~\citep{anonymous2025embedding} approaches the CPO~\citep{achiam2017constrained} update for $\beta\to 0$. 
%Put differently, while P3O approximates CPO, C3PO approximates C-TRPO using the same exact penalty approach. 
Further, C3PO is conceptually similar to P2BPO~\citep{dey2024p2bpo}, in using a more conservative version of the P3O loss, but C3PO does not use a penalty with a fixed location at the constraint, but a moving penalty which recedes as the iterate gets closer to the constraint. This allows C3PO to approach the optimal feasible solution without regularization bias.
%Intuitively, C3PO and P2BPO would bes equivalent, if the barrier penalty in P2BPO would recede es the iterate moves closer to the constraint, i.e. if P2BPO's $\beta$-parameter would increase. 
%Practically, this receding is the essence of the central path property of C-TRPO and C-NPG.

\section{Computational Experiments}\label{sec:emperiments}
To evaluate our approach, we conduct experiments aimed at testing the benefits of using central path approximation as a design principle for constrained policy optimization algorithms.\footnote{Code: \url{https://github.com/milosen/c3po}}
We benchmark C3PO against a range of representative constrained reinforcement learning baselines. 
We include methods from three major algorithmic families: penalty-based methods (P3O, P2BPO), Lagrangian methods (PPO-Lag, CPPO-PID), and trust-region methods (CPO, C-TRPO).

Conceptually, \emph{penalty-based methods}, especially algorithms that augment the PPO loss with a penalty, like P3O~\citep{zhang2022penalizedproximalpolicyoptimization} and P2BPO~\citep{dey2024p2bpo}, are closest to our approach. 
Like C3PO, those penalize constraint violations directly in the policy gradient loss using a ReLU-penalty.
\emph{Lagrangian methods} maintain dual variables to enforce constraints adaptively. 
PPO-Lagrangian~\citep{ray2019benchmarking} applies this principle to the PPO algorithm, forming a loss which is similar to C3PO's. 
For completeness, we consider CPPO-PID~\citep{stooke2020responsivesafetyreinforcementlearning} as a more recent Lagrangian baseline.
Finally, \emph{trust region methods}, such as CPO~\citep{achiam2017constrained} and C-TRPO~\citep{anonymous2025embedding}, use trust regions and constrained updates to maintain stable reward improvement and feasibility throughout training.
They do not aim for scalability, but form strong baselines on the benchmarks.

We benchmark the algorithms on 4 locomotion tasks and 4 navigation tasks from Safety Gymnasium \citep{ji2023safety}, as done by \cite{anonymous2025embedding}.
For the baseline algorithms, we use the hyper-parameters reported by \cite{ji2023safety}, and for P3O and C-TRPO the recommended parameters in \cite{zhang2022penalizedproximalpolicyoptimization} and \cite{anonymous2025embedding} respectively. 
For C3PO we use $\kappa=30.0$ and $w=0.05$. Each algorithm is trained on each task for 10 million steps with a cost threshold of 25.0. Final iterate performance is measured by aggregating over 5 seeds using \texttt{rliable}~\citep{agarwal2021deep}.
\begin{wrapfigure}{rt}{0.6\textwidth}
  \begin{center}
\includegraphics[width=0.58\textwidth]{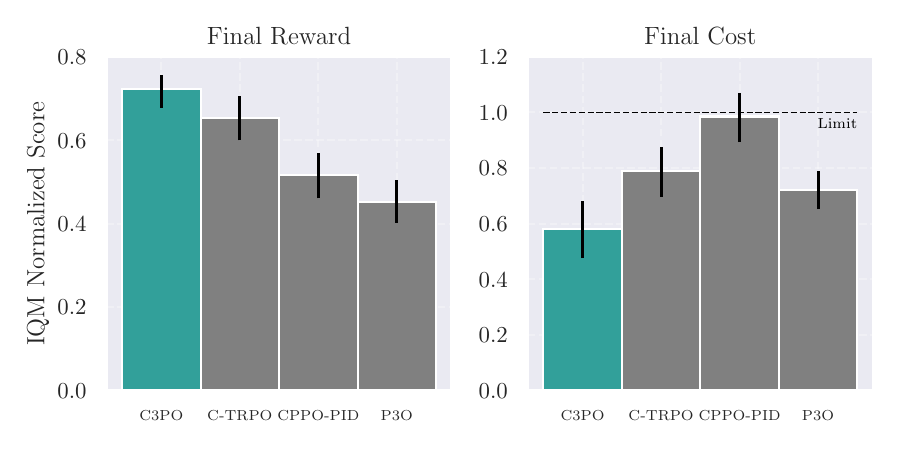}
  \end{center}
  \caption{Aggregated performance using the inter quartile mean (IQM) across 8 tasks from Safety Gymansium for a subset of algorithms. The algorithms were chosen as the feasible representatives of their respective group.}
  \label{fig:benchmark-small}
  \vspace{-10pt}
\end{wrapfigure}

The results provide confirmatory evidence for the usefulness of the central path approach. Policies trained with C3PO exhibit a stable progression toward the constrained optimum, maintaining feasibility for most training iterations, see Figure \ref{fig:lag-failure-racecar}.
Furthermore, C3PO consistently outperforms prior PPO-style penalty methods in terms of achieved reward, while also adhering more strictly to the specified constraints, see Figure \ref{fig:benchmark-small}.
This improved trade-off between reward and feasibility offers additional support for the effectiveness of the central path approach.
While C3PO does not outperform trust-region methods across all tasks in the benchmark, it performs well consistently, resulting in high aggregated performance. 
The full benchmark results table and more examples like Figure \ref{fig:lag-failure-racecar} are presented in Appendix~\ref{app:experiments}.
\begin{figure}[t]
    \centering
    \includegraphics[width=\linewidth]{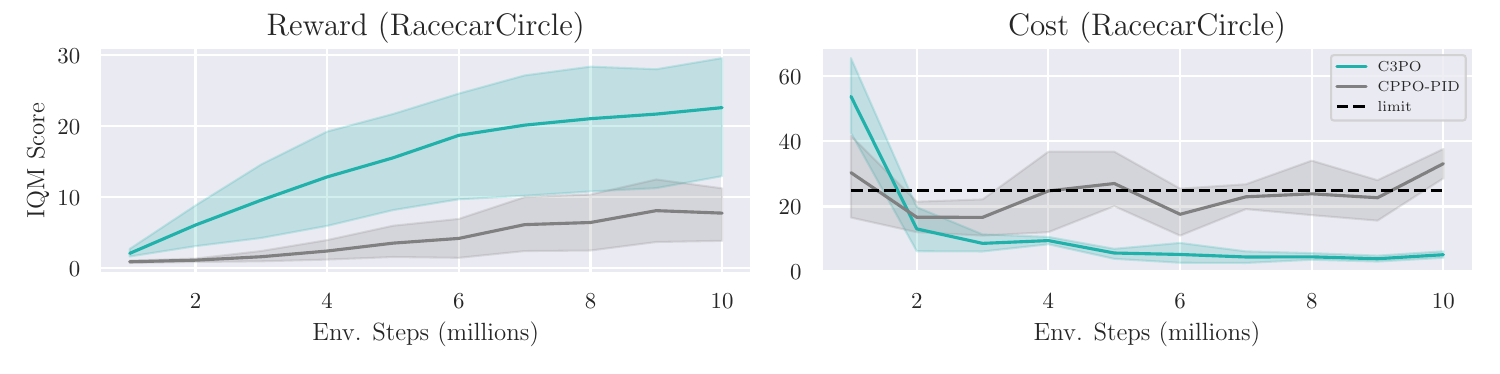}
    \caption{Example of improved performance through approximating the central path: Lagrangian methods tend to converge less reliably towards a safe policy and oscillate around the constraint. This does not yield a higher final reward.
    Instead, staying feasible from early on in training seems to have a positive effect on final reward.}
    \label{fig:lag-failure-racecar}
\end{figure}

\section{Conclusion}
In this work, we use central path approximation as a guiding principle for designing policy optimization methods for constrained RL. 
We propose C3PO, an algorithm which is obtained through a simple augmentation of the original PPO-loss inspired by the central path approach.
Our experimental results support this design principle: 
Compared to existing PPO-style penalty and Lagrangian methods, C3PO exhibits improved performance with tighter constraint satisfaction, highlighting the benefits of a central path approach in constrained policy optimization.

While the current results are limited to small-scale simulations and simplified settings, such as a single constraint per task, they suggest that central path approximation is a promising design principle for constrained RL algorithms. 
%We hope this early-stage contribution encourages further discussion and refinement of on-policy safe policy optimization algorithms and other methods that solve constrained MDPs.
We hope this encourages further research, following this paradigm. %improving constrained RL algorithms using insights from constrained optimization.
Future directions include the extensions to high-dimensional tasks, % LLM finetuning
theoretical guarantees,  % see \cite{ohara2007information}
and applications such as safety-critical control and LLM fine-tuning. 

\begin{ack}
N. M. and N.S. are supported by BMBF (Federal Ministry of Education and Research) through ACONITE (01IS22065) and the Center for Scalable Data Analytics and Artificial Intelligence (ScaDS.AI.) Leipzig and by the European Union and the Free State of Saxony through BIOWIN. N.M. is also supported by the Max Planck IMPRS CoNI Doctoral Program.
\end{ack}

\bibliographystyle{plainnat}
\bibliography{exait_2025}

%%%%%%%%%%%%%%%%%%%%%%%%%%%%%%%%%%%%%%%%%%%%%%%%%%%%%%%%%%%%

\appendix

\section{Extended Background}\label{app:background}
\subsection{The Geometry of Policy Optimization}
\cite{neu2017unified} have shown that the policy divergence used to define the trust-region in TRPO~\cite{schulman2017trust} can be derived as the Bregman divergence generated by a mirror function on the state-action polytope. 
TRPO's mirror function is the negative conditional entropy
\begin{equation}
    \Phi_K(\rho) = \sum_{s,a}\rho(s,a)\log\pi_\rho(a|s)
\end{equation}
which generates
\begin{equation}
    D_K(\pi_k||\pi) = \sum_{s,a}\rho_k(s,a)[\log\pi(a|s)-\log\pi_k(a|s)]
\end{equation}
via the operator
\begin{align}
    D_\Phi(x||y) \coloneqq \Phi(x) - \Phi(y) - \nabla \Phi(y)^\top(x-y).
\end{align}

In general, a trust region update is defined as
\begin{align}\label{app:eq:TRPO}
    \pi_{k+1} \in \argmax_{\pi \in \Pi %
    } %\Big\{
    \mathbb{A}_{r}^{\pi_k}(\pi) %: \theta\in\Theta 
    \quad \text{ sbj. to } 
    D_\Phi(\rho_{\pi_k}||\rho_{\pi}) \le \delta
    %\Big\}
    ,
\end{align}
where $D_\Phi\colon\mathcal{K}\times\mathcal{K}\to\mathbb R$ is the \emph{Bregman divergence} induced by a suitably convex function $\Phi\colon \operatorname{int}(\mathcal{K}) \rightarrow \mathbb{R}$.

\subsection{The Safe Geometry Approach}
\cite{anonymous2025embedding} consider mirror functions of the form
\begin{align}
    \Phi_{\operatorname{C}}(\rho) &\coloneqq\Phi_\textup{K}(\rho) + \sum_i \beta_i\Phi_\textup{B}(\rho)\\
                               &\coloneqq\sum_{s,a}\rho(s,a)\log\pi_\rho(a|s)  + \sum_{i=1}^m \beta_i \phi\left(b_i-\sum_{s,a}\rho(s,a)c(s,a)\right), 
\end{align}
where $\rho\in\mathcal{K}_\textup{safe}$ is a feasible state-action occupancy, $\Phi_{\operatorname{K}}$ is the negative conditional entropy, and $\phi$ is convex. 
Further, $\phi\colon\mathbb R_{>0}\to\mathbb R$  with $\phi'(x)\to+\infty$ for $x\searrow0$. 
The log-barrier $\phi(x)=-\log(x)$ considered in this work is a possible candidate.
In general, the induced divergence takes the form
\begin{align}\label{eq:weighted-kl}
    %\begin{split}
        D_{\C}(\rho_{1}||\rho_{2}) 
    &= D_\textup{K}(\rho_{1}||\rho_2)
        + \sum_{i=1}^m \beta_i D_\textup{B}(\rho_{1}||\rho_{2})\\
    &= D_{\operatorname{K}}(\rho_{1}||\rho_2)
        + \sum_{i=1}^m \beta_i [\phi(b_{1;i}) - \phi(b_{2;i})
     + \phi'(b_{2;i})C_i(\pi_1) - \phi'(b_{2;i})C_i(\pi_2))],
    %\end{split}
\end{align}
where $b_{\pi;i}=d_i - C_i(\pi)$. 
The corresponding trust-region scheme is
\begin{align}\label{eq:Safe-TRPO}
    \pi_{k+1} \in \arg \max_{\pi \in \Pi
    } \mathbb{A}_{r}^{\pi_k}(\pi)
    \quad \text{ sbj. to } 
    D_\C(\rho_{\pi_k}||\rho_{\pi}) \le \delta.
\end{align} 
Analogously to the case of unconstrained TRPO, there is a corresponding natural policy gradient scheme:
\begin{align}\label{eq:NPG-safe}
    \theta_{k+1} = \theta_k + \epsilon_k G_\C(\theta_k)^+ \nabla R(\theta_k),
\end{align}
where $G_\C(\theta)^+$ denotes an arbitrary pseudo-inverse of the Gramian
\begin{align*}
    G_\C(\theta)_{ij} = \partial_{\theta_i}\rho_\theta^\top \nabla^2 \Phi_\C(\rho_\theta) \partial_{\theta_j} \rho_\theta. 
\end{align*}

The authors discuss that, under suitable parametrizations of $\theta \mapsto \pi$, this gradient preconditioner is a Riemannian metric on $\Theta_\textup{safe}$ and natural policy gradient flows based on $G_\C(\theta_k)$ leave $\Theta_\textup{safe}$ invariant. Further, $G_\C(\theta_k)^+$ is equivalent to the Hessian of $D_\C$:

\begin{align*}
    H_\C(\theta) & = \mathbb{E}_{s\sim \rho_\theta}F(\theta) + \sum_i \beta_i \phi''(b_i-C_i(\theta)) \nabla_\theta^2 C_i(\theta) \Big|_{\theta = \theta_k}.
\end{align*}

where $F$ is the fisher information of the policy.
Unlike in TRPO, the divergence itself is not easy to estimate, however, the authors demonstrate that another divergence has the same Hessian, i.e. is equivalent up to second order in the policy parameters. It is derived using a ``surrogate advantage trick'' for $C_i$ and results in the divergence
\begin{align}
    \bar D_{KL}(\pi,\pi_k) + \beta\bar D_\Phi(\pi,\pi_k) = \bar D_{KL}(\pi,\pi_k) + \beta \cdot [\phi(b_k-\mathbb{A}_{c}^{\pi_k}(\pi)) - \phi(b_k) - \phi'(b_k)\cdot\mathbb{A}_{c}^{\pi_k}(\pi)],
\end{align}
which is ultimately used as a drop-in replacement for the conventional divergence in TRPO.

\subsection{Central Paths}
In the small step size limit, the trajectories induced by trust region methods converge to the corresponding natural policy gradient (NPG) flow on the state-action polytope $\mathcal{K}$. 
The space of state-action occupancies $\rho\in\mathcal{K}$ forms not only a polytope, but a Hessian manifold~\cite{muller2023geometry}. 
C-TRPO induces such a gradient flow on the LP Equation \ref{eq:CMDP-lin} w.r.t the Hessian geometry induced by the convex function 
\begin{equation}
\Phi(\rho) = \sum_{s,a}\rho(s,a)\log \pi(a|s) - \beta\log(\rho - \sum_{s,a}\rho(s,a)c(s,a)). 
\end{equation}
It is well known that Hessian gradient flows $(\rho_t)$ of linear programs follow the central path, meaning that they are characterized as the optimizers of regularized linear programs with regularization strength $t^{-1}$. 
In policy space, we obtain for a single constraint
\begin{equation}
    \pi_t = \argmax\{ R(\pi) + t^{-1} D_{\Phi}(\pi, \pi_0) : \pi, C(\pi) \le d \}.
\end{equation}
Since $\Phi$ curves infinitely towards the boundary of the feasible set of LP Equation \ref{eq:CMDP-lin}, solving the problem posed by C-TRPO corresponds to solving LP Equation \ref{eq:CMDP-lin} using an interior point / barrier method with barrier $D_{\Phi}(\cdot, \pi_0)$.
For a more detailed discussion of Hessian geometries and natural policy gradients see \cite{alvarez2004hessian, muller2023geometry, muller2024essentially}. 

\section{Proofs of Section 3}\label{app:derivation}

\subsection{Exact Penalty Methods}

We provide a general result for the exactness of the penalties considered in this work. For general discussions of exact penalty methods, we refer to standard textbooks in optimization~\cite{bertsekas1997nonlinear, nocedal1999numerical}. 
Here, we consider a compact subset $X\subset\mathbb{R}^n$ with non-empty interior, differentiable functions $f,g\in C^1(X)$, and the constrained optimization problem
\begin{align}\label{app:eq:constrained}
    \max f(x) \quad \text{subject to } g(x)\le b,
\end{align}
where we impose Slater's condition $\{ x\in X : g(x) <  b \}\ne\varnothing$ to be non-empty and $f$ to be concave and $g$ to be convex. 
We denote the penalized functions by 
\begin{align}
    P_\kappa(x) \coloneqq f(x) - \kappa\max\{0, g(x)-b\}. 
\end{align}
Recall the definition of the Lagrangian 
\begin{align}
    \mathcal L(x,\lambda)=f(x)-\lambda(g(x)-b). 
\end{align}

\begin{theorem}[Exactness for convex programs]\label{app:thm:exactness}
Assume that there exists a solution $x^\star\in X$ of \eqref{app:eq:constrained} and denote the corresponding dual variable by $\lambda^\star\ge0$. 
%Then the following holds: 
%\begin{enumerate}
%    \item 
For $\kappa>\lambda^\star$ we have  
    \begin{align}
        \arg\max \{f(x) : x\in X, g(x)\le b \} = \arg\max \{P_\kappa(x) : x\in X\}.
    \end{align}
%    \item For $\kappa<\lambda^\star$ we have 
%    \begin{align}
%        \arg\max \{f(x) : x\in X, g(x)\le b \} \ne \arg\max \{P_\kappa(x) : x\in X\}.
%    \end{align}
%    \item If moreover $\nabla g$ is $L_g$\nobreakdash‑Lipschitz near $x^\star$ and $x(\kappa)\in\arg\max P_\kappa$, then with $\Delta=\max\{0,g(x(\kappa))-b\}$ one has
%    \begin{align}
%        \Delta\le\frac{\lambda^\star-\kappa}{L_g}, \quad \text{and } f^\star-f(x(\kappa))\le(\lambda^\star-\kappa) \Delta \le \frac{(\lambda^\star-\kappa)^2}{L_g}.
%    \end{align}
%\end{enumerate}
\end{theorem}
\begin{proof}
    %\begin{enumerate}
        %\item  
        Consider an infeasible point $\bar x\in X$ of $P_\kappa$, meaning that $g(\bar x)>b$. 
        Note by convexity $x^\star$ maximizes the Lagrangian $\mathcal L(\cdot, \lambda^\star)$. 
        Then
        \begin{equation*}
            P_\kappa(\bar x)
            =f(\bar x)-\kappa\,(g(\bar x)-b)
            < f(\bar x)-\lambda^\star(g(\bar x)-b)
            =\mathcal L(\bar x,\lambda^\star)
            \le\mathcal L(x^\star,\lambda^\star)
            =P_\kappa(x^\star). 
        \end{equation*}
        Hence, every maximizer of $P_\kappa$ is feasible and thus a solution of the regularized problem, showing the inclusion $\supseteq$.   
        As $P_\kappa$ agrees with $f$ for feasible points, we also obtain that every maximizer of $f$ over the feasible set is a maximizer of $P_\kappa$. 
        
        %\item Assume now that $\lambda^\star>0$ and fix $\kappa\in(0,\lambda^\star)$. 
        %By complementary slackness, we have $g(x^\star) = b$. 
        %a small perturbation of $x^\star$ into $g>b$ 

        %\item 
    %\end{enumerate}
\end{proof}

\subsection{C3PO Exact Penalty}\label{app:c3po-derivation}

\barrierVSlin*
\begin{proof}
%First, we aim to show that the constraints $D_\textup{B}(\pi,\pi_k)<\delta_\textup{B}$ and $\mathbb{A}_c^{\pi_k}(\pi)-w_\textup{B}\cdot b<0$ are equivalent for some $w_\textup{B}$ when $d-C(\pi_k) > \mathbb{A}_c^{\pi_k}(\pi)\geq0$.

Let
$$P_\textup{Barrier}\coloneqq\{\pi: D_\textup{B}(\pi,\pi_k)\leq\delta_B,\ 
 \mathbb{A}_c^{\pi_k}(\pi)\geq0\}$$
and
$$P_\textup{Lin}\coloneqq\{\pi:\mathbb{A}_c^{\pi_k}(\pi) -w\cdot b\leq0,\ 
 \mathbb{A}_c^{\pi_k}(\pi)\geq0\}.$$
 Note that 
 \begin{equation}
    D_\textup{B}(\pi,\pi_k) = \frac{b-\mathbb{A}_{c}^{\pi_k}(\pi)}{b} - \log\left(\frac{b-\mathbb{A}_{c}^{\pi_k}(\pi)}{b}\right) - 1
\end{equation}
is a strictly convex increasing function of $\mathbb{A}_c^{\pi_k}$ for $\mathbb{A}_c^{\pi_k}\geq0$ (see Figure \ref{fig:moving-barrier}).
This means that there exists a unique $\mathbb{A}_\textup{B}>0$ that solves
\begin{equation}\label{eq:div-equality}
 \frac{b-\mathbb{A}_\textup{B}}{b} - \log\left(\frac{b-\mathbb{A}_\textup{B}}{b}\right) -1 = \delta_\textup{B}
\end{equation}
and for $\mathbb{A}_\textup{B} \geq \mathbb{A}_c^{\pi_k}(\pi)>0$ it holds that $\delta_\textup{B} \geq D_\textup{B}(\pi,\pi_k)>0$.
To solve for $\mathbb{A}_\textup{B}$, we rewrite \ref{eq:div-equality} as
\begin{equation}
 \left(\frac{\mathbb{A}_\textup{B}-b}{b}\right)\exp\left(\frac{\mathbb{A}_\textup{B}-b}{b}\right) = -\exp(-\delta_\textup{B} - 1).
\end{equation}
and use the definition of Lambert's W-Function~\cite{corless1996lambert} to invert the left hand side as follows
\begin{equation}
    \frac{\mathbb{A}_\textup{B}-b}{b} = W(-\exp(-\delta_\textup{B} - 1)),
\end{equation}
where $W$ is the real part of the principle branch of the W-Function.
Finally, rearranging yields
\begin{equation}
    \mathbb{A}_\textup{B} = b\cdot(W(-\exp(-\delta_\textup{B} - 1))+1).
\end{equation}

Note that $b>\mathbb{A}_\textup{B}>0$ must still hold.
With this result,
\begin{align}
    P_\textup{Barrier}&=\{\pi: \mathbb{A}_c^{\pi_k}(\pi)-\mathbb{A}_\textup{B}<0,\ 
 \mathbb{A}_c^{\pi_k}(\pi)\geq0\},\\
 &=\{\pi: \mathbb{A}_c^{\pi_k}(\pi)-b\ (W(-\exp(-\delta_\textup{B} - 1))+1)<0,\ 
 \mathbb{A}_c^{\pi_k}(\pi)\geq0\},\\
 &=\{\pi: \mathbb{A}_c^{\pi_k}(\pi)-b\ w<0,\ 
 \mathbb{A}_c^{\pi_k}(\pi)\geq0\},
\end{align}
 showing that $P_\textup{Barrier}=P_\textup{Lin}$ for a unique $w$.

Further, since $\min(b,w\cdot b) = w\cdot b$ for $b>0$, the solution sets of
\begin{align}
\max_{\pi \in \Pi}\ \mathbb{A}_{r}^{\pi_k}(\pi) \ &\textrm{ s.t. }\ \mathbb{A}_c^{\pi_k}(\pi) - \mathbb{A}_\textup{B}<0
            \ &\textrm{ and }\ \bar D_\textup{KL}(\pi,\pi_k) <\delta_\textup{KL}\\
\max_{\pi \in \Pi}\ \mathbb{A}_{r}^{\pi_k}(\pi) \ &\textrm{ s.t. }\ D_\textup{B}(\pi,\pi_k)<\delta_B
            \ &\textrm{ and }\ \bar D_\textup{KL}(\pi,\pi_k) <\delta_\textup{KL}
\end{align}
agree for $\mathbb{A}_c^{\pi_k}(\pi)\geq0$ and $w=W(-\exp(-\delta_\textup{B} - 1))+1$.

Finally, by theorem \ref{app:thm:exactness}, they must also agree with the solutions of
\begin{align}
\max_{\pi \in \Pi}\ \mathbb{A}_{r}^{\pi_k}(\pi) &-\kappa_k \max\{0, \mathbb{A}_{c}^{\pi_k}(\pi) - \min(b,w\cdot b)\}\ &\textrm{ s.t. }\ \bar D_\textup{KL}(\pi,\pi_k)<\delta_\textup{KL}, \\
\max_{\pi \in \Pi}\ \mathbb{A}_{r}^{\pi_k}(\pi) &-\kappa_k \max\{0, D_\textup{B}(\pi,\pi_k)-\delta_B\}\ &\textrm{ s.t. }\ \bar D_\textup{KL}(\pi,\pi_k)<\delta_\textup{KL},
\end{align}
under the same conditions and for large enough $\kappa$.
\end{proof}

Note that the cost budget $b=d-C(\pi_k)$ is multiplied with a fixed function of $\delta_\textup{B}$. 
Hence, we can use $w$ as the hyper-parameter immediately instead of defining it through $\delta_\textup{B}$.

%%%%%%%%%%%%%%%%%%%%%%%%%%%%%%%%%%%%%%%%%%%%%%%%%%%%%%%%%%%%

%%% END INSTRUCTIONS %%%
\clearpage
\section{Experiment Details}\label{app:experiments}
\begin{figure}
    \centering
    \includegraphics[width=1\linewidth]{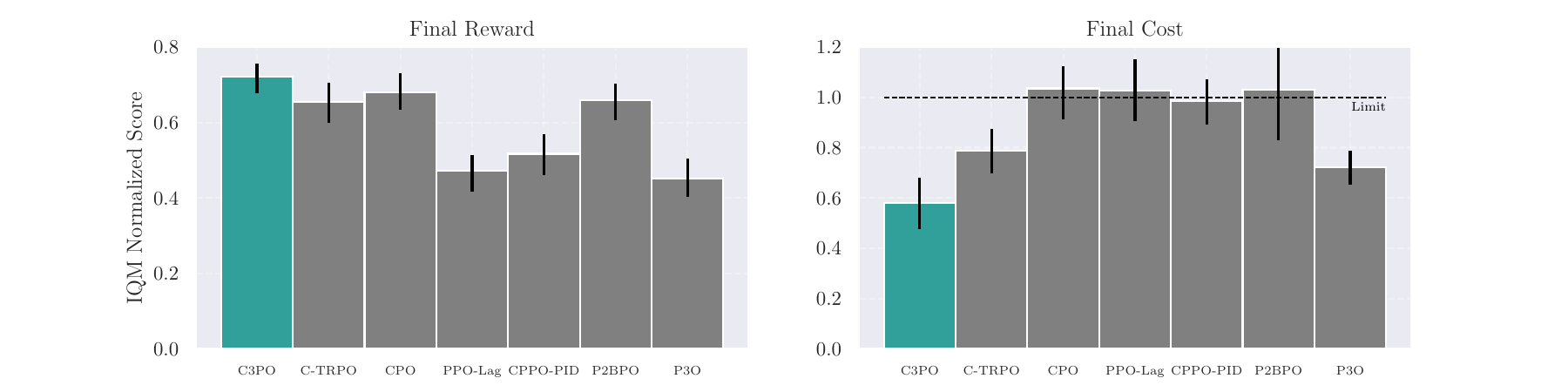}
    \caption{Aggregated performance using the inter quartile mean (IQM) across 8 tasks from Safety Gymansium for all algorithms (except P2BPO) across 8 tasks. P2BPO has been excluded, since the final cost (right) was off the charts. This may be due to the missing penalty coefficient in the algorithm.}
    \label{fig:enter-label}
\end{figure}
\begin{figure}[ht]
    \centering
    \includegraphics[width=\linewidth]{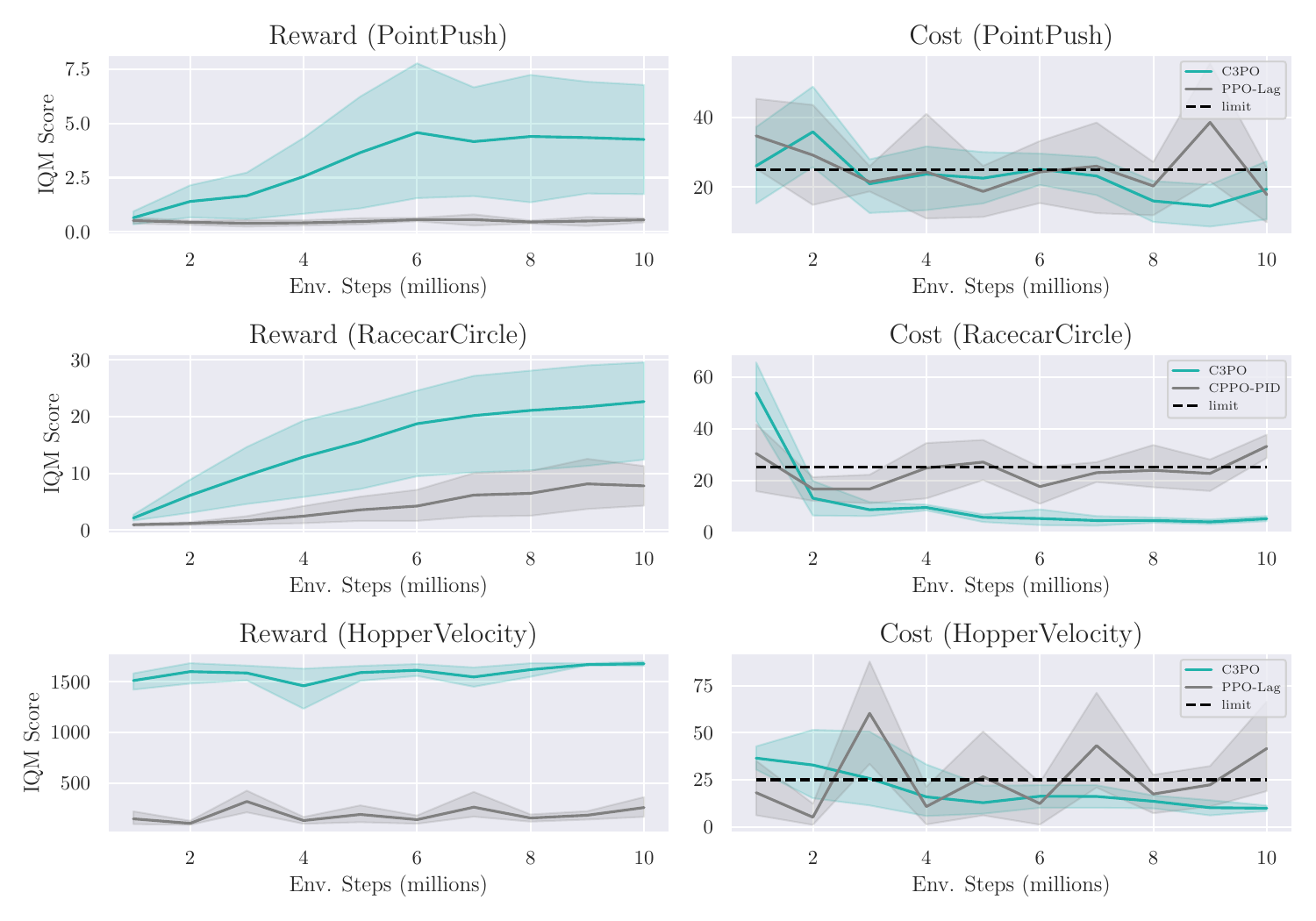}
    \caption{Hand-picked examples where central path approximation improves final reward performance.}
    \label{fig:lag-failure}
\end{figure}

\begin{table}
\caption{Performance of 8 representative safe policy optimization algorithms on 8 tasks from Safety Gymnasium for 10 million steps and a cost threshold of 25.0 aggregated over 5 seeds each. Bold marks the algorithm with the highest mean cumulative reward among the admissible ones. An algorithm is admissible, if its average cumulative cost achieved at the end of training is below the threshold.}
\tiny
\setlength\tabcolsep{2pt}
\centering
\begin{tabular}{llllllllll}
\toprule
 &  & Ant & HalfCheetah & Humanoid & Hopper & CarButton & PointGoal & RacecarCircle & PointPush \\
\midrule
\multirow[t]{2}{*}{C3PO} & $R$ & {3043 ± 44} & {2458 ± 436} & {5389 ± 93} & {\bfseries 1674 ± 35} & {2.3 ± 0.7} & {23.8 ± 0.9} & {25.9 ± 5.1} & {\bfseries 4.5 ± 2.6} \\
 & $C$ & 15.0 ± 4.7 & 13.3 ± 6.4 & 1.2 ± 0.9 & 9.9 ± 1.7 & 53.4 ± 22.3 & 37.9 ± 1.7 & 5.0 ± 1.7 & 20.2 ± 10.0 \\
\cline{1-10}
\multirow[t]{2}{*}{C-TRPO} & $R$ & {3019 ± 149} & {2841 ± 41} & {5746 ± 248} & {1621 ± 82} & {1.1 ± 0.2} & {\bfseries 19.3 ± 0.9} & {29.5 ± 3.1} & {1.0 ± 6.6} \\
 & $C$ & 13.2 ± 9.2 & 12.1 ± 7.6 & 12.2 ± 5.9 & 17.7 ± 8.0 & 34.0 ± 10.2 & 23.3 ± 3.6 & 20.2 ± 4.0 & 25.3 ± 7.0 \\
\cline{1-10}
\multirow[t]{2}{*}{CPO} & $R$ & {3106 ± 21} & {2824 ± 104} & {5569 ± 349} & {1696 ± 19} & {1.1 ± 0.2} & {20.4 ± 2.0} & {\bfseries 29.8 ± 1.9} & {0.7 ± 2.9} \\
 & $C$ & 25.1 ± 11.3 & 23.1 ± 8.0 & 16.2 ± 8.6 & 25.7 ± 4.4 & 33.5 ± 8.7 & 28.2 ± 4.1 & 23.1 ± 4.5 & 28.9 ± 20.0 \\
\cline{1-10}
\multirow[t]{2}{*}{PPO-LAG} & $R$ & {3210 ± 85} & {\bfseries 3033 ± 1} & {5814 ± 122} & {240 ± 159} & {0.3 ± 0.8} & {9.4 ± 1.8} & {30.9 ± 1.8} & {0.6 ± 0.0} \\
 & $C$ & 28.9 ± 8.7 & 23.2 ± 1.9 & 12.7 ± 31.0 & 38.8 ± 36.4 & 39.2 ± 41.1 & 22.5 ± 10.1 & 31.7 ± 2.7 & 18.2 ± 9.5 \\
\cline{1-10}
\multirow[t]{2}{*}{CPPO-PID} & $R$ & {3205 ± 76} & {3036 ± 10} & {\bfseries 5877 ± 84} & {1657 ± 61} & {\bfseries -1.2 ± 0.6} & {6.1 ± 4.8} & {8.1 ± 4.3} & {1.0 ± 1.1} \\
 & $C$ & 26.2 ± 4.4 & 26.5 ± 7.2 & 20.3 ± 6.0 & 18.6 ± 8.1 & 23.8 ± 6.0 & 21.8 ± 6.8 & 33.3 ± 5.9 & 22.8 ± 9.9 \\
\cline{1-10}
\multirow[t]{2}{*}{P2BPO} & $R$ & {3269 ± 18} & {2928 ± 46} & {5293 ± 171} & {1573 ± 85} & {6.1 ± 0.9} & {25.9 ± 0.2} & {15.7 ± 7.5} & {1.1 ± 0.5} \\
 & $C$ & 32.3 ± 8.9 & 26.0 ± 19.7 & 1.5 ± 1.1 & 13.2 ± 11.7 & 125 ± 14 & 39.6 ± 5.7 & 5.5 ± 8.0 & 43.8 ± 28.9 \\
\cline{1-10}
\multirow[t]{2}{*}{P3O} & $R$ & {\bfseries 3122 ± 24} & {3020 ± 12} & {5492 ± 118} & {1633 ± 49} & {0.2 ± 0.3} & {5.7 ± 0.3} & {0.9 ± 0.1} & {0.7 ± 0.6} \\
 & $C$ & 21.2 ± 2.5 & 27.0 ± 1.1 & 4.2 ± 2.2 & 14.6 ± 1.6 & 40.9 ± 18.2 & 17.1 ± 6.2 & 13.1 ± 4.6 & 14.1 ± 9.4 \\
\cline{1-10}
\multirow[t]{2}{*}{PPO} & $R$ & {5402 ± 274} & {6583 ± 954} & {6138 ± 699} & {1810 ± 390} & {18.2 ± 1.2} & {26.6 ± 0.2} & {40.8 ± 0.5} & {0.9 ± 0.7} \\
 & $C$ & 887 ± 27 & 976 ± 1 & 783 ± 60 & 435 ± 85 & 378 ± 18 & 50.7 ± 3.3 & 200 ± 4 & 42.9 ± 24.0 \\
\cline{1-10}
\bottomrule
\end{tabular}
\end{table}

\end{document}